\def\year{2021}\relax
\documentclass[letterpaper]{article} 
\usepackage{aaai21}  
\usepackage{times}  
\usepackage{helvet} 
\usepackage{courier}  
\usepackage[hyphens]{url}  
\usepackage{graphicx} 
\usepackage{natbib}  
\usepackage{caption} 
\nocopyright
\usepackage{amsthm}
\usepackage{amsmath,amssymb,amsfonts}
\usepackage{mathtools}
\usepackage{dsfont}

\usepackage{algorithm}
\usepackage{algorithmic} 

\usepackage{xcolor}

\usepackage{xspace}
\newcommand{\latinphrase}[1]{\textit{#1}}

\newcommand{\ie}{\latinphrase{i.e.}\xspace}
\newcommand{\eg}{\latinphrase{e.g.}\xspace}

\newcommand{\oursys}{\textit{\textrm{Dopamine}}\xspace}
\newcommand{\user}{patient\xspace}
\newcommand{\users}{patients\xspace}
\newcommand{\agent}{hospital\xspace}
\newcommand{\agents}{hospitals\xspace}
\newcommand{\server}{server\xspace}
\newcommand{\norm}[1]{\left\lVert#1\right\rVert_2}

\newcommand{\SOne}{{\textsf{C}}\xspace} 
\newcommand{\STwo}{{ \textsf{CDP}}\xspace}
\newcommand{\SThree}{{ \textsf{F}}\xspace}
\newcommand{\SFour}{{\textsf{FPDP}}\xspace}

\newtheorem{definition}{Definition}

\newtheorem{lem}{Lemma}

\title{Dopamine: Differentially Private Federated Learning on Medical Data}
\author{Mohammad Malekzadeh, Burak Hasircioglu, Nitish Mital, Kunal Katarya,\\ Mehmet Emre Ozfatura, Deniz Gündüz\thanks{\scriptsize{This work was funded by the European Research Council~(ERC) through Starting Grant BEACON~(no. 677854) and by the UK EPSRC~(grant no. EP/T023600/1).}}\\
}
\affiliations{
    Department of Electrical and Electronic Engineering,
    Imperial College London.\\
    \{m.malekzadeh, b.hasircioglu18, n.mital, kunal.katarya15, m.ozfatura, d.gunduz\}@imperial.ac.uk
}

\usepackage{blindtext,graphicx}
\usepackage[absolute]{textpos}
\begin{document}

\begin{textblock}{10}(1,1)
\noindent The Second AAAI Workshop on Privacy-Preserving Artificial Intelligence (PPAI-21)
\end{textblock}

\maketitle

\begin{abstract} 
While rich medical datasets are hosted in \agents distributed across the world, concerns on \users' privacy is a barrier against using such data to train deep neural networks~(DNNs) for medical diagnostics.  We propose \oursys, a system to train DNNs on distributed datasets, which employs federated learning~(FL) with differentially-private stochastic gradient descent~(DPSGD), and, in combination with secure aggregation, can establish a better trade-off between differential privacy~(DP) guarantee and DNN's accuracy than other approaches. Results on a diabetic retinopathy~(DR) task show that \oursys provides a DP guarantee close to the centralized training counterpart, while achieving a better classification accuracy than FL with parallel DP where DPSGD is applied without  coordination. Code is available  at~{\url{https://github.com/ipc-lab/private-ml-for-health}}. 
\end{abstract} 

\section{Introduction}
Deep neural networks facilitate disease recognition from medical data, particularly for \users without immediate access to doctors. Medical images are processed with DNNs for faster diagnosis of skin disease~\cite{skincancer}, lung cancer~\cite{lungcancer}, or diabetic retinopathy~(DR)~\cite{10.1001/jama.2016.17216}. However, the memorization capacity of DNNs can be exploited by adversaries for reconstruction of a \user's data~\cite{carlini2019secret}, or the inference of a \user's participation in the training dataset~\cite{shokri2017membership, dwork2017exposed}. Due to such privacy risks and legal restrictions, medical data can rarely be found in one centralized dataset; thus,
there has been a surge of interest in {\em privacy} and {\em utility} preserving training on distributed medical datasets~\cite{kaissis2020secure}.

{\em Federated learning}~(FL)~\cite{mcmahan2017communication} trains a DNN where \agents collaborate with a central \server in training a global model on their local datasets. At each round, the \server sends the current model to each \agent, then \agents update the model on their private datasets and send the model back to the \server. The \agents' updates are  susceptible to information leakage about the \users' data due to model over-fitting to training data~\cite{carlini2019secret}. {\em Differential privacy}~(DP)~\cite{dwork2014algorithmic} limits an adversary's certainty in inferring a \user's presence in the training dataset.  Before optimizing the DNN, Gaussian random noise is added to the computed gradients on the \users' data to achieve {differentially-private stochastic gradient descent}~(DPSGD)~\cite{abadi2016deep}. 

We propose \oursys, a customization of DPSGD for FL, which, in combination with secure aggregation by homomorphic encryption, can establish a better privacy-utility trade-off than the existing approaches, as elaborated in Section~\ref{sec_evaluation}. Experimental results on the DR dataset~\cite{retina.1}, using SqueezeNet~\cite{iandola2016squeezenet} as a benchmark DNN, show that \oursys can achieve a DP bound close to the centralized training counterpart, while achieving better classification accuracy than FL with parallel DP where the \agents apply typical DPSGD on their sides without any specific coordination. We provide theoretical analysis for the guaranteed privacy by \oursys, and discuss the differences between \oursys and the seminal work proposed by~\cite{truex2019hybrid}: our solution allows to properly keep track of the privacy loss at each round as well as taking advantage of the {\em momentum}~\cite{qian1999momentum} in FL-based DPSGD, which improves the DNN's accuracy. 

The main contribution of this paper is the design and implementation of FL on medical images while satisfying record-level DP. While previous works on medical datasets~(as discussed in Appendix,  Section~\ref{sec_background}) either do not guarantee a formal notion of privacy, \eg \cite{li2020multi}, or apply weaker notion of DP, \eg parameter-level DP~\cite{li2019privacy}, to the best of our knowledge, \oursys is the first system that implements FL-based DPSGD that guarantees record-level DP for a dataset of medical images. Finally, we publish a simulation environment to facilitate further research on privacy-preserving training on distributed medical image datasets.

\section{\oursys's Methodology}~\label{sec:our system}

{\bf Problem Formulation\footnote{\scriptsize{As notation, we use lower-case {\it italic}, \eg $x$, for variables; upper-case {\it italic}, \eg $X$, for constants; bold font, \eg $\mathbf{X}$, for vectors, matrices, and tensors; blackboard font, \eg $\mathbb{X}$, for sets; and calligraphic font, \eg $\mathcal{X}$, for functions and algorithms. 
}}} Let each \agent $k \in \{1,\ldots,K\}$ own a dataset, $\mathbb{D}_k$, with an unknown number of \users, where each \user $i$ participates with a labeled data $(\mathbf{x}_i, y_i)$. The global \server owns a validation dataset $\mathbb{D}_{G}$, and we consider a DNN's {\em utility} as its prediction accuracy on $\mathbb{D}_{G}$. The goal is to collaboratively train a DNN  while satisfying record-level DP. We assume a {\em threat model} where \users only trust their local \agent, and \agents are non-malicious and non-colluding. The global \server is honest but curious. Finally, \agents and the \server do not trust any other third parties. We assume the \user's {\em privacy}, defined by $(\epsilon, \delta)$, is a bound on the record-level DP loss. The \agents aim to ensure their \users a computational DP against the \server during training, and an information-theoretical DP against the \server and any other third parties after training. It is called computational DP as a cryptosystem is only robust against computationally-bounded adversaries. \oursys assumes that adversaries are computationally bounded during training, which is a typical assumption, but after training they can be computationally unbounded\footnote{\scriptsize{Background 
materials are provided in Appendix, Section~\ref{sec_background}.}}.     

\subsection{\oursys's Training}\label{sec_training}
The training procedure is given in Algorithm \ref{alg:oursys}, where we perform federated SGD among $K$ \agents. At each round $t$, each \agent $k$ samples a batch of samples from its local dataset, $\mathbb{D}_k^t \subset \mathbb{D}_k$, where each local sample is chosen independently and with probability $q$. Due to independent sampling, the batch size is not fixed, and is a binomial random variable with parameters $q$ and $|\mathbb{D}_k|$. Let $|\mathbb{D}_k^t|$ and $\eta$ denote the batch size and the learning rate, respectively. Let $C$ denote the maximum value of the $L_2$-norm of per-sample gradients. If a per-sample gradient has a norm greater than this, then its norm is clipped to $C$~\cite{abadi2016deep}. 

\begin{algorithm}[t!]
\caption{\oursys's Training}\label{alg:oursys}
\begin{algorithmic}[1]
\STATE {\bfseries Input:} $K$: number of \agents, $\mathbb{D}$: distributed dataset, $\mathbf{w}$: model's trainable parameters, $\mathcal{L}(\cdot, \cdot)$: loss function, $q$: sampling probability, $\sigma$: noise scale, $C$: gradient norm bound, $\eta$: learning rate, $\beta$: momentum, $T$: number of rounds, $(\epsilon, \delta)$: bounds on record-level DP loss.
\STATE {\bfseries Output:} $\mathbf{w}_{G}$: optimized global model.
\STATE $\mathbf{w}^{0}_{G} =$ random initialization.
\STATE $\hat{\epsilon} = 0$ 
\FOR{$t: 1,\ldots, T$} 
    \FOR{$k: 1, \ldots, K$} 
        \STATE $\mathbb{D}^{t}_k = Sampling(\mathbb{D}_k)$ \quad // by uniformly sampling each item in $\mathbb{D}_k$ independently with probability $q$.
        \FOR{$ \mathbf{x}_i \in \mathbb{D}^{t}_k$}
            \STATE $\mathbf{g}^t(\mathbf{x}_i) = \nabla_{\mathbf{w}}\mathcal{L}(\mathbf{w}^{t-1}_{G},\mathbf{x}_i)$
            \STATE $\bar{\mathbf{g}}^t(\mathbf{x}_i) = \mathbf{g}^t(\mathbf{x}_i)/\max{\big(1, \frac{||\mathbf{g}^t(\mathbf{x}_i)||_2}{C}\big)}$ 
        \ENDFOR
        \STATE$\widetilde{\mathbf{g}}^t_k = \frac{1}{|\mathbb{D}^{t}_k|} \big( \sum_{\mathbf{x}_i \in \mathbb{D}^{t}_k} \bar{\mathbf{g}}^t(\mathbf{x}_i) + \mathcal{N}(0,\frac{\sigma^2\cdot C^2\cdot \mathbf{I}}{K}) \big)$
        \STATE $\hat{\mathbf{g}}^t_k = \widetilde{\mathbf{g}}^t_k + \beta \hat{\mathbf{g}}^{t-1}_k$ \quad //$\hat{\mathbf{g}}^{0}_k = 0$  
        \STATE $\mathbf{w}^{t}_{k} = \mathbf{w}^{t-1}_{G}-\eta \hat{\mathbf{g}}^t_k$ 
    \ENDFOR
    \STATE $\hat{\epsilon} = CalculatePrivacyLoss(\delta, q, \sigma, t)$ // by Moments Accountant \cite{abadi2016deep}  
    \IF{ $\hat{\epsilon} > \epsilon$} 
        \STATE return ${\mathbf{w}}^{t-1}_{G}$
    \ENDIF
    \STATE $\mathbf{w}^{t}_{G}= \frac{1}{K}\big( SecureAggregation(\sum_{k}  \mathbf{w}^{t}_{k}$) \big)
    \STATE $Broadcast(\mathbf{w}^{t}_{G})$
\ENDFOR
\end{algorithmic}
\end{algorithm}

As \agents do not trust the \server, a potential solution is, for each \agent, to add a large amount of noise to the model updates, $\mathbf{w}^{t}_{k}$, to keep them differentially private from the \server. However, adding a large amount of noise has the undesirable effect of decreasing the accuracy of the trained model. A better solution is to employ secure aggregation of the model updates, which prevents the \server from discovering the \agent's model updates. Since the model updates are now hidden from the \server, the \agents can add less amount of noise to keep their model updates differentially private from the \server.
\begin{lem}
\label{secure_agg_lemma}
In Algorithm~\ref{alg:oursys}, if each \agent $k$ adds Gaussian noise $n_k\sim\mathcal{N}\Big(\mu=0,\sigma^2=\frac{2\ln{(1.25/\delta)}C^2}{\epsilon^2 |\mathbb{D}_k^t|^2 K}\Big)$ to the average of (clipped) gradients, then $\mathbf{w}^{t}_{G}$ is $(\epsilon,\delta)$-DP against the \server, and $\big(\epsilon \sqrt{{K}/{K-1}},\delta\big)$-DP against any \agent. 
\end{lem}

\begin{proof} 
We compute the effect of the presence or the absence of a single sample at \agent $k$ on $\mathbf{w}_{G}^{t}$.  Let $\mathbf{w}_k^t-\mathbf{w}_G^{t-1}$ be the (noiseless) model update of \agent $k$ at round $t$. Due to secure aggregation, the \server and other \agents receive~(unencrypted) $\mathbf{w}_G^{t}-\mathbf{w}_G^{t-1} = \frac{1}{K} \sum_{k=1}^{K}
\mathbf{w}_k^t - \mathbf{w}_G^{t-1}$. Since $\mathbf{w}_G^{t-1}$ is already known to all, it is sufficient to only consider $\mathbf{w}_k^t$. We define $\mathbf{w'}_k^t$ as the local model if $\mathbb{D'}_k^t$ is used by \agent $k$ instead of $\mathbb{D}_k^t$, where $\mathbb{D}_k^t$ and $\mathbb{D'}_k^t$ only differ in one sample.
Since $L_2$-norm of each per-sample gradient is bounded by $C$ and the \agent $k$ averages all per-sample gradients of $|\mathbb{D}_k^t|$ samples, we have $$\max_{\mathbb{D}_k^t,\mathbb{D'}_k^t}\norm{\mathbf{w}_k^t-\mathbf{w'}_k^t}={\eta C}/{|\mathbb{D}_k^t|},$$ for all $k \in [K]$. As adding or removing one sample at \agent $k$ only changes the gradients of that \agent, we have $$\max_{\mathbb{D}_k^t,\mathbb{D'}_k^t}\norm{\mathbf{w}_G^{t}-\mathbf{w'}_G^{t}}=\frac{1}{K}\max_{\mathbb{D}_k^t,\mathbb{D'}_k^t}\norm{\mathbf{w}_k^t-\mathbf{w'}_k^t}=\frac{\eta C}{|\mathbb{D}_k^t|K},$$ for all $k \in [K]$, where $\mathbf{w'}_G^{t}$ is the global model when $\mathbb{D'}_k^t$ is used.
Thus, to guarantee a record-level $(\epsilon,\delta)$-DP, the variance of effective noise added to each model update must be 
$$\sigma^2_\text{effective} = \frac{2\ln(1.25/\delta)\eta^2 C^2}{\epsilon^2 K^2 |\mathbb{D}_k^t|^2}.$$
If hospitals use noise $n_k$ in Lemma~\ref{secure_agg_lemma}, then the effective noise added to the global model update, $\mathbf{w}_G^{t}-\mathbf{w}_G^{t-1}$, is $\frac{\eta}{K}\sum_{k=1}^K n_k$, with a variance $\frac{\eta \sigma^2}{K}$. When the expression for $\sigma^2$ in Lemma~\ref{secure_agg_lemma} is put in place, the variance of the effective noise is $\frac{2\ln(1.25/\delta)\eta^2 C^2}{\epsilon^2|\mathbb{D}_k^t|^2 K^2} =\sigma^2_\text{effective}$, which proves \mbox{$(\epsilon,\delta)$-DP} against the \server. From a \agent's perspective, since it knows its share in the effective noise, the variance of the noise after canceling its share is $\frac{2\ln(1.25/\delta)\eta^2 C^2(K-1)}{\epsilon^2|\mathbb{D}_k^t|^2 K^3}$, which results in $(\epsilon \sqrt{K/(K-1)},\delta)$-DP against \agents.
\end{proof}

\subsection{Single vs. Multiple Local Updates}
An algorithm, similar to our Algorithm~\ref{alg:oursys}, is previously proposed by~\cite{truex2019hybrid}~(Section 5.2), where they allow each party to carry out multiple local SGD steps before sharing the updated model with the \server. We argue that this approach has an important drawback: it may violate the critical assumption in the {\em moments accountant}~(MA) procedure~\cite{abadi2016deep} for tracking the privacy loss at each round, and for the same reason, it does not allow using momentum for SGD~(line 13 in our Algorithm~\ref{alg:oursys}), which helps to improve the model's accuracy.  The MA is introduced by~\cite{abadi2016deep} for keeping track of a bound on the moments of the {\em privacy loss random variable}, in the sense of DP, that depends on the random noise added to the algorithm's output. In the proof provided by~\cite{abadi2016deep}, the MA is updated by sequentially applying DP mechanism $\mathcal{M}$. Particularly, \cite{abadi2016deep} models the system by letting the  mechanism at round $t$, $\mathcal{M}^t$, to have access to the output of all the previous mechanisms $\mathcal{M}^{t-1}, \mathcal{M}^{t-2}, ..., \mathcal{M}^{1}$, as the auxiliary input. Hence, to properly calculate the privacy loss at round $t$, one needs to make sure that the amount of privacy loss in the previous $t-1$ rounds are properly bounded by adding the proper amount of noise. 

In Algorithm~\ref{alg:oursys}, we allow each \agent to add a Gaussian noise of variance $\sigma^2/{K}$, at each round, and then share the updated model to securely perform FedSGD~\cite{mcmahan2017communication}. Thus, before running any other computation on the DNN, we add  the required amount of effective noise that is needed by MA to calculate the privacy loss $\epsilon$~(see Lemma~\ref{secure_agg_lemma}).
Moreover, the effect of each record on the local momentum is the same as its effect on the aggregated model, thus, due to the post-processing property, using momentums does not incur a privacy cost~(see Appendix~\ref{sec:momnt}).
However, \cite{truex2019hybrid} allow the local model to be updated for multiple~(\eg 100) local iterations while at each iteration the DNN model that is used as the input is the output of the previous local iterations, and not the aggregated model. Thus, the sequential property of the proof in MA may be violated. Moreover, when running multiple local iterations, it is not clear how one can ensure a proper sampling similar to what is needed in the central DPSGD.

Fixing these issues, our algorithm allows to take advantage of both privacy loss tracking with the MA and using SGD with momentums. We emphasize that although our algorithm needs more communications due to sharing model updates after every iteration, this cost can be tolerated considering the importance of the other two aspects in medical data processing: model accuracy and privacy guarantee.

\section{Evaluation}\label{sec_evaluation}
\subsection{Dataset}\label{sec_datasets}
 Diabetic retinopathy is an eye condition that can cause vision loss in diabetic patients. It is diagnosed by examining scans of the blood vessels of a patient's retina - thereby making this an image classification task. A DR dataset~\cite{retina.1}, available at \url{https://www.kaggle.com/c/aptos2019-blindness-detection/data}, exists to solve this task. The problem is to classify the images of \users' retina into five categories: {\em No DR}, {\em Mild DR}, {\em Moderate DR}, {\em Severe DR}, and {\em Proliferative DR}. The dataset consists of 2931 variable-sized images for training, and 731 imgaes for testing. To deal with the variation in image dimensions, all images were resized into the same dimensions of $224\times224$.

\subsection{DNN Architecture}

We use SqueezeNet~\cite{iandola2016squeezenet}: a convolutional neural network that has 50x fewer parameters than the famous AlexNet, but is shown to achieve the same level of accuracy on the ImageNet dataset. The size of the model is a very important factor for both training and inference, which is one of the main motivations for using SqueezeNet. Moreover, SqueezeNet can achieve a classification accuracy of about 80\% on the DR dataset while the current best accuracy reported on the Kaggle's competition\footnote{\scriptsize{\url{https://www.kaggle.com/c/aptos2019-blindness-detection/notebooks}}} on this dataset is about 83\% using a much larger DNN, EfficientNet-B3, that has 18x more parameters than SqueezeNet\footnote{\scriptsize{The Pytorch implementation of is available at \url{https://pytorch.org/docs/stable/_modules/torchvision/models/squeezenet.html}}}.


\subsection{Baselines}~\label{sec_evaluation_methods}
There are several approaches for training a DNN that provide us different trade-offs between utility, privacy, and complexity. We compare \oursys against the  following well-known approaches:

    \subsubsection{Non-Private Centralized Training~(\SOne).} Every \agent shares their data with the \server, such that a centralized dataset $\mathbb{D}$ is available for training. In this method the problem is reduced to architecture search and optimization procedure for training a DNN on $\mathbb{D}$. This method provides the best utility, a moderate complexity, but no privacy. It results in a single point of failure in the case of data breach. Moreover, in situations where \users are from different countries or private organizations, is almost impossible reaching an agreement on a single trusted curator. 
    
    \subsubsection{Centralized Training with DP~(\STwo).} When a trusted curator is available, one can trade some utility in the method~\SOne to guarantee central DP for the \users by using a differentially private training mechanism such as DPSGD~\cite{abadi2016deep}. The complexity of ~\STwo is similar to~\SOne, except for the fact that existing DP mechanisms make the training of DNNs slower and less accurate, as they require performing additional tasks, such as per-sample gradient clipping and noise addition. 
    
    \subsubsection{Non-Private FL~(\SThree).} Considering that a trusted curator is not realistic in many scenarios due to legal, security, and business constraints, \agents can collaborate via FL, instead of directly sharing their \users' sensitive data. In~\SThree, every \agent has its own dataset that may differ in size and quality. We assume the \server runs FedAvg~\cite{mcmahan2017communication}. While~\SThree introduces some serious communication complexities, it removes the need for trusted curator and enables an important layer of privacy protection. Basically, after sharing the data in the methods~\SOne and \STwo, \agents have no more control on the type and amount of information that can be extracted from their \users' data. However, in~\SThree, \agents do not share the original data, but some aggregated statistics or the results of a computation over the dataset. Even more, \agents can decide to stop releasing information at any time, thus having more control over the type and amount of information sharing. Although it seems that the utility of~\SThree might be lower than that of~\SOne, there are studies~\cite{bonawitz2019towards} showing that with a more sophisticated algorithms, one can achieve the same utility as~\SOne.
    
    \subsubsection{FL with Parallel DP~(\SFour).} Similarly to~\STwo, one can apply a DP mechanism to the local training procedure at the \agent side. Thus, every \agent locally provides a central DP guarantee. Notice that this is different form a local DP guarantee~\cite{dwork2014algorithmic}, as each \agent is a trusted curator. In~\SFour, a much better privacy can be provided than the previous methods. However, as the size of each local dataset $\mathbb{D}^i$ is much smaller than the size of the whole dataset $\mathbb{D}$, more noise should be added to achieve the same level of privacy, which results in a utility loss.

\subsection{Results}~\label{sec_evaluation_result}

\oursys proposes FL with a Central DP guarantee that is similar to~\SThree, except the \agents add noise to their gradients before running a secure aggregation procedure to aggregate the models at the \server. \oursys is similar to \SFour in the sense that the \agents add noise to their gradients, but thanks to secure aggregation, the variance of the noise \agents add is less than \SFour by an order of $K$. 

Figure~\ref{fig:dr_dataset_1} compares the classification accuracy~(top plot) and DP bound~(bottom plot) of all the introduced methods for SqueezeNet trained on the DR dataset. We show the mean and standard deviation of these 5 runs for each method. To be fair, for each method we tuned the hyper-parameters to achieve the best accuracy of that method. For FL scenarios, we have divided the dataset equally among 10 \agents~($K=10$) in an independent and identically distributed~(iid) manner, thus each \agent owns 293 images. For \SThree and  \SFour we perform FedAvg~\cite{mcmahan2017communication}, where each \agent performs 5 local epochs in between each global epoch, and at each round half of the \agents are randomly chosen to participate.  The details of the hyper-parameters used are provided at \url{https://github.com/ipc-lab/private-ml-for-health/tree/main/private_training}/

We see that \oursys's bound is very close to that of the centralized training counterpart \STwo, and even better than \STwo at early epochs. \oursys achieves a better classification accuracy than \SFour and more importantly much better record-level privacy guarantee. In \SFour, each hospital needs to add noise with a variance $K$ times larger to get the same privacy level of \STwo. In our experiments, we found that this amount of noise makes learning impossible. Finally, we see that \SThree, compared to \SOne, has a competitive performance, but there still is a considerable accuracy gap between \SThree and \oursys due to privacy protection provided. Such a gap can be mitigated if we could get access to a larger dataset and proportionally reduce the amount of noise without weakening the privacy protection.  

\begin{figure}[t!]
    \centering
    \begin{minipage}[t]{\linewidth}
    \includegraphics[width=\columnwidth]{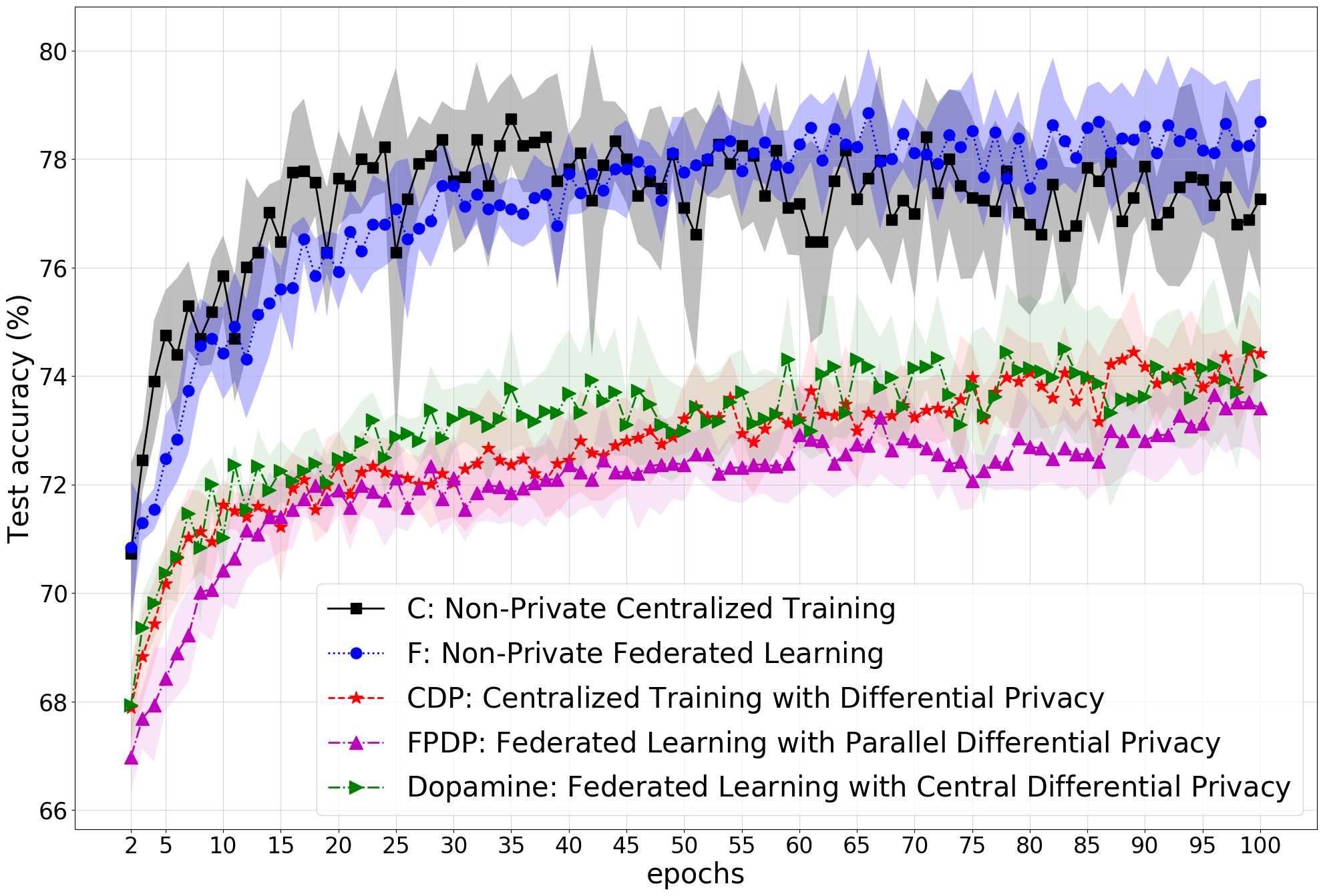}
    \end{minipage}
    \vfill
    \begin{minipage}[t]{\linewidth}
\includegraphics[width=\columnwidth]{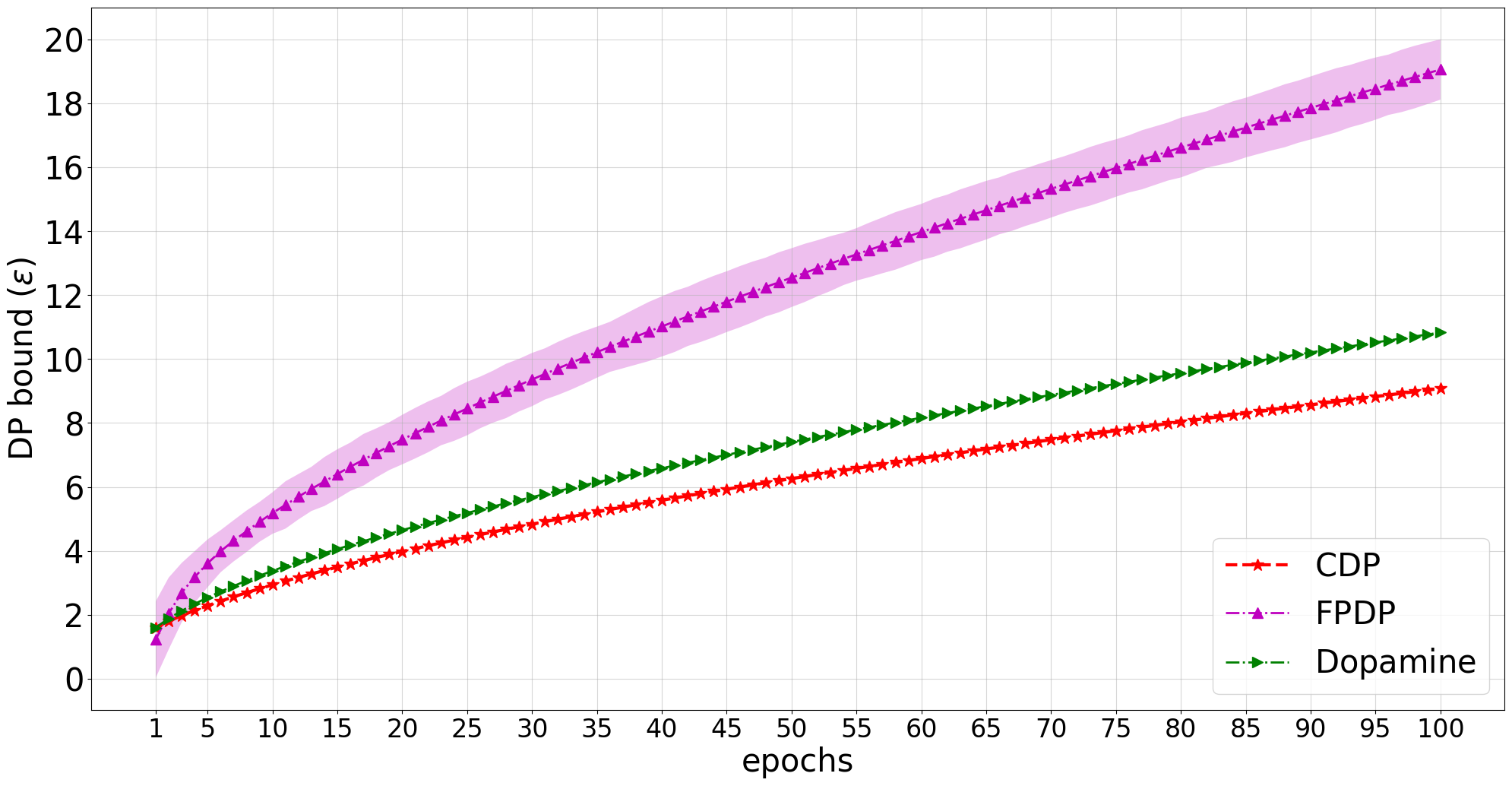}
    \end{minipage}
    \caption{Comparison of (top)~classification accuracy (bigger better) vs. (bottom)~DP bound (smaller better) on DR dataset. \oursys achieves a DP bound very close to the centralized training counterpart, while it also achieves a better classification accuracy than \SFour. For all DP methods, $\delta=10^{-4}$. For \oursys, each epoch is equal to one round (see Figure~\ref{fig:fig_dr_dataset_2} in Appendix for a more details)}  
    \label{fig:dr_dataset_1}
    \vspace{-10pt}
\end{figure}

\section{Conclusion and Future Work}\label{sec_conclusion}
We proposed \oursys, a system for collaboratively training DNNs on private medical images distributed across several \agents. Experiments on a benchmark dataset of DR images show that \oursys can outperform existing baselines by providing a better utility-privacy trade-off.

The important future directions are: (i) we aim to combine the proposed secure aggregation functionality~(discussed in Appendix~\ref{sec:secure_agg}) to the \oursys's training pipeline. (ii) Current privacy-analysis only allows us to have one local iteration. We observed that allowing local \agents to perform more than one iteration can improve the accuracy by about 3\%. We will carry out the privacy analysis for this case in the future. (iii) Existing DP libraries do not allow using arbitrary DNN architectures. They are either not compatible with more recent architectures, like EfficientNet, or they need too much resources. Making DP training faster is still an open area of research and engineering. (iv) Finally, we aim to evaluate \oursys on other medical datasets, and investigate novel methods for accuracy improvement while retaining the privacy protection provided.


\appendix

\section{Related Work}\label{sec_related_work}

\cite{sheller2018multi} apply FL, without any privacy guarantee,  for the segmentation of brain tumor images using a deep convolutional neural network, known as U-Net. They compare FL with a baseline collaborative learning scenario where \agents iteratively train a shared model in turn, and show that FL outperforms this baseline.

\cite{pfohl2019federated} employ FL with FedAvg and DPSGD algorithms to train  logistic regression and feedforward networks with one hidden layer for in-hospital mortality and prolonged length of stay prediction using the patients electronic health records. Comparing FL with a centralized approach, when DP is used in both cases, their experimental results show that, while centralized training can achieve a good DP bound with $\epsilon\approx 1$, FL with DP performs poorly in terms of both accuracy and $\epsilon$. 

Considering the differences in type of the device, number of sensors, and the placement of the sensors that is used in each \agent  for collecting electroencephalography~(EEG) data, \cite{gao2019hhhfl} propose an application of FL for training a classifier over heterogeneous data; however, they do not consider any formal privacy guarantee.  \cite{li2020multi}  use FL among \agents for training an fMRI image classifier and add some noise to the trained parameters before sharing them with the \server which does not follow the requirement of a valid DP mechanism, thus it does not provide a formal privacy guarantee.

\cite{choudhury2020anonymizing} offer a k-anonymity based privacy guarantee to generate a syntactic tabular dataset, instead of the original dataset, and use it for participating in FL. \cite{li2019privacy} propose an FL system for brain tumour segmentation in a method where each \agent owns a dataset of MRI scans. However, their implementation does not satisfy record-level DP, but parameter-level DP, which for neural networks is not a meaningful privacy protection.

\section{Background}\label{sec_background} 
\subsection{Differential Privacy (DP)}

When we want to publish the results of a computation  $\mathcal{F}(\mathbb{D})$, over a private dataset $\mathbb{D}$, DP helps us to bound the ability of an adversary in distinguishing whether any specific sample data was included in the dataset $\mathbb{D}$ or not. DP has a worst-case threat model, assuming that all-but-one can collude and adversaries can have access to any source of side-channel information. 

\begin{definition}[\textbf{Differential Privacy}]\label{def_dp}
Given $\epsilon, \delta \geq 0$, a mechanism (\ie algorithm) $\mathcal{M}$ satisfies $(\epsilon, \delta)-$differential privacy if for all pairs of neighboring datasets $\mathbb{D}$ and $\mathbb{D}'$ differing in only one sample, and for all $\mathbb{Y} \subset Range(\mathcal{M})$, we have
\begin{equation}\label{eq_def_dp}
\Pr\bigl(\mathcal{M}(\mathbb{D}) \in \mathbb{Y} \bigr) 
\leq e^{\epsilon} \Pr\bigl(\mathcal{M}(\mathbb{D}') \in \mathbb{Y}\bigr) + \delta,
\end{equation} 
where $\epsilon$ is the parameter that specifies the privacy loss~(\ie $e^\epsilon$), $\delta$ is the probability of failure in assuring the upper-bound on the privacy loss, and the probability distribution is over the internal randomness of the mechanism $\mathcal{M}$, holding the dataset fixed~{\cite{dwork2006calibrating, dwork2014algorithmic}}. When $\delta = 0$, it is called {\em pure} DP, which is stronger but less flexible in terms of the mechanism design.
\end{definition}

For approximating a deterministic function $\mathcal{F}$, such as the mean or sum, a typical example of $\mathcal{M}$ in Equation~(\ref{eq_def_dp}) is a zero-mean Laplacian or Gaussian noise addition mechanism~\cite{dwork2014algorithmic}, where the variance is chosen according to the {\em sensitivity} of $\mathcal{F}$ that is the maximum of the absolute distance $|\mathcal{F}(\mathbb{D}) - \mathcal{F}(\mathbb{D}')|$ among all pairs of neighboring datasets $\mathbb{D}$ and $\mathbb{D}'$~\cite{dwork2006calibrating}. In general, to be able to calculate the sensitivity of $\mathcal{F}$, we need to know, or to bound, the $Range(\mathcal{F})$; otherwise we cannot design a reliable DP mechanism.

\begin{definition}[\textbf{Gaussian Mechanism}]\label{gaussian_mech}
Gaussian mechanism adds a Gaussian noise to a function $\mathcal{F}(x)$ so that the response to be released is $(\epsilon,\delta)$-differentially private. It is defined as $$\mathcal{M}(x) = \mathcal{F}(x) + \mathcal{N}\Big(\mu=0,\sigma^2=\frac{2\ln{(1.25/\delta)}(\Delta_2 \mathcal{F})^2}{\epsilon^2}\Big),$$
where $\Delta_2\mathcal{F}$ is the $L_2$ sensitivity of $\mathcal{F}(x)$ to the neighboring datasets~\cite{dwork2014algorithmic}. 

\end{definition}

DP mechanisms are in two types: central DP and local DP. In central DP, we consider a model where there is a trusted \server, thus the \users original data is shared with that \server. The trusted \server then run $\mathcal{M}$ on the collected data before sharing the result with other parties. In local DP, every \user randomly transforms data at the \user's side without needing to trust the \server or other parties~\cite{Erlingsson2014, wang2017locally, erlingsson2019amplification}.  Thus, in terms of privacy guarantee, local DP ensures that the probability of discovering the true value of the user's shared data is limited to a mathematically defined upper-bound. However, central DP gives such an upper-bound for the probability of discovering whether a specific user has shared their data or not.

DP mechanisms are mostly used for computing aggregated statistics, such as the sum, mode, mean, most frequent, or histogram of the dataset~\cite{bittau2017prochlo, ding2017collecting, cormode2018privacy}. DP mechanisms can also be used for training machine learning models on sensitive datasets while providing DP guarantees for the people that are included in the training dataset~\cite{abadi2016deep, bonawitz2019towards, kifer2020guidelines}. 

\subsection{Federated Learning (FL)}
Many parameterized machine learning problems can be modeled as a {\em stochastic optimization problem}
\begin{equation}\label{fl_problem}
\min_{\mathbf{w}\in\mathbb{R}^{d}} \mathrel{\mathop:}= \mathds{E}_{\mathbb{D} \sim \mathcal{D}}\mathcal{L}(\mathbf{w}, \mathbb{D} ),
\end{equation}
where $\mathbf{w}$ denotes the model parameters,  $\mathbb{D}$ is a random dataset sampled from the true, but unknown, data distribution $\mathcal{D}$, and $\mathcal{L}$ is a task-specific empirical loss function.

The core idea of the FL is to solve the stochastic optimization problem in~(\ref{fl_problem}) in a distributed manner such that the participating users do not share their local dataset with each other but the model parameters to seek a consensus. Accordingly, the objective of $K$ users that collaborate to solve Equation~(\ref{fl_problem}) is reformulated as the sum of $K$ loss functions defined over different sample datasets, where each corresponds to a different user:
\begin{equation} \label{fl_problem_dist}
\min_{\mathbf{w}} \mathcal{L}(\mathbf{w})= \frac{1}{K}\sum^{K}_{k=1} {\mathds{E}_{\mathbb{D}_k \sim \mathcal{D}_{k}}\mathcal{L}(\mathbf{w},\mathbb{D}_k)},
\end{equation}
where $\mathcal{D}_{k}$ denotes the portion of the dataset allocated to user $k$. SGD is a common approach to solve machine learning problems that are represented in the form of Equation~(\ref{fl_problem}). FL considers the problem in Equation~(\ref{fl_problem_dist}), and aims to solve it using SGD in a decentralized manner.

\subsection{Homomorphic Encryption (HE)}

{\em Fully Homomorphic Encryption}~(FHE) allows computation of arbitrary functions $\mathcal{F}$ on encrypted data, thus enabling a plethora of applications like private computation offloading. Gentry~\cite{10.5555/1834954} was the first to show that FHE is possible with a method that includes the following steps: (i) construct a somewhat homomorphic encryption~(SHE) scheme that can evaluate functions of low degree, (ii) simplify the decryption circuit, and (iii) evaluate the decryption circuit on the encrypted ciphertexts homomorphically to obtain new ciphertexts with a reduced inherent noise. The third step is called bootstrapping, and this allows computations of functions of arbitrary degree. Here, the degree of a function is the number of operations that must be performed in sequence to compute the function. For example, to compute $x^2$, one multiplication is necessary, therefore its degree is $1$, while to compute $x^3$, 2 multiplications ($x*x$ and $x*x^2$) are necessary, and therefore its degree is $2$. 

The security of homomorphic encryption schemes is based on the learning with errors  problem~\cite{10.1145/1568318.1568324}, or its ring variant called ring learning with errors~\cite{10.1145/2535925}, the hardness of which can be shown to be equivalent to that of classical hard problems on ideal lattices, which are the basis for a lot of post-quantum cryptography. All FHE schemes add a {\em small noise} component during encryption of a message with a {\em public key}, which makes the decryption very hard unless one has access to a {\em secret key}. The decrypted message obtained consists of the message corrupted by a small additive noise, which can be removed if the noise is ``small enough''. Each computation on ciphertexts increases the noise, which ultimately grows large enough to fail the decryption. The bootstrapping approach is used to lower the noise in the ciphertext to a fixed level. Addition of ciphertexts increases the noise level very slowly, while multiplication of ciphertexts increases the noise very fast. When FHE is only used for aggregation, the bootstrapping step, which is a computationally expensive procedure, is not usually necessary since the noise never grows to a large enough size.

While machine learning models compute on floating point parameters, encryption schemes work only on fixed point parameters, or equivalently, integer parameters, which can be obtained by appropriately scaling and rounding fixed point values. Given an integer $u$ to encode an integer base $b$, a base-b expansion of $u$ is computed, and represented as a polynomial with the same coefficients. The expansion uses a `balanced' representation of integers modulo $b$ as the coefficients, that is, when $b$ is odd, the coefficients are chosen to lie in the range $\left[\frac{-(b-1)}{2},\frac{(b-1)}{2}\right]$, and when $b$ is even, the coefficients are chosen to lie in the range $\left[ \frac{-b}{2}, \frac{(b-1)}{2} \right]$, except if $b=2$, when the coefficients are chosen to lie in the set $\{0,1 \}$. For example, if $b=2$, the integer $26 = 2^4 + 2^3 + 2^1 $ is encoded as the polynomial $x^4 + x^3 + x$. If $b=3$, $26 = 3^3 - 3^0$ is encoded as the polynomial $x^3 - 1$. Decoding is done by computing the polynomial representation at $x=b$. To map an integer to a polynomial ring $R=\mathbb{Z}[x]/f(x)$, where $f(x)$ is a monic polynomial of degree $d$, the polynomial representation of the integer is viewed as being a plaintext polynomial in $R$. The base $b$ is called the \textit{plaintext modulus}. The most popular choice of $f(x)$ is $x^d + 1$, with $d=2^n$ called the \textit{polynomial modulus}. Let $q>1$ be a positive integer, then we denote the set of integers $[-\frac{q}{2}, \frac{q}{2}]$ by $\mathbb{Z}_q$, and the set of polynomials in $R$ with coefficients in $\mathbb{Z}_q$ by $R_q$. The integer $q$ is called the \textit{coefficient modulus}, and is usually set much larger than the plaintext modulus.

For multi-dimensional vectors, the single-instruction-multiple-data technique, also known as {\em batching} \cite{Dorz2016OntheflyHB}, encodes a given array of multiple integers into a single plaintext polynomial in $R_q$. A single computation instruction on such a plaintext polynomial is equivalent to simultaneously executing that instruction on all the integers in that array, thus speeding up the computation by many orders of magnitude. The plaintext modulus is assumed to be larger than any integer in the given array. The length of the array is assumed to be equal to the degree of the polynomial modulus $d$, and in case that its size is less than $d$, then it is padded with zeros to make its length equal to $d$. Lagrange interpolation is performed over the array of integers to obtain a polynomial of degree $d-1$, thus obtaining the plaintext polynomial in $R_q$ encoding $d$ integers.

\section{On the Effect of Momentums}\label{sec:momnt}
\
Consider the momentum term~(Algorithm~\ref{alg:oursys}, line 13) at round $t$. For user $k$, we have
\begin{equation}
\hat{g}_k^t=\tilde{g}_k^t + \beta\hat{g}_k^{t-1}=\sum_{i=0}^{t-1}\beta^i\tilde{g}_k^{t-i},
\end{equation}
 where $\tilde{g}_k^t$'s are the noisy gradients~(Algorithm~\ref{alg:oursys}, line 12).

After every round, $\hat{g}_k^t$'s are securely averaged at the \server, such that the averaged gradient at round $t$ becomes 
\begin{equation}
    \hat{g}^t = \frac{1}{K}\sum_{k=1}^K\hat{g}_k^t=\sum_{k=1}^K \sum_{i=0}^{t-1}\beta^i\tilde{g}_k^{t-i}=\frac{1}{K}\sum_{i=0}^{t-1}\beta^i \sum_{k=1}^K \tilde{g}_k^{t-i}.
\end{equation}
Since the term $\sum_{k=1}^K \tilde{g}_k^{t-i}$ is computed securely, and the noise of each $\tilde{g}_k^{t-i}$ adds up, as it is shown in Lemma~\ref{secure_agg_lemma}, $\forall i \in \{0, 1, \dots,t-1\}$, $\tilde{g}^{t-i}$ is $({\epsilon},\delta)$-DP against the \server and $({\epsilon}\sqrt{K/(K-1)},\delta)$-DP against other \agents, where $\hat{\epsilon} \leq \epsilon$ is the DP-bound at round $i$.
 
Thanks to the post-processing property~\cite{dwork2014algorithmic} of DP, use of the same noisy gradients repetitively does not incur any additional privacy costs. Thus, use of momentum in our setting does not incur additional privacy cost either.

\begin{figure*}[t!]
    \centering
    \includegraphics[width=.8\textwidth]{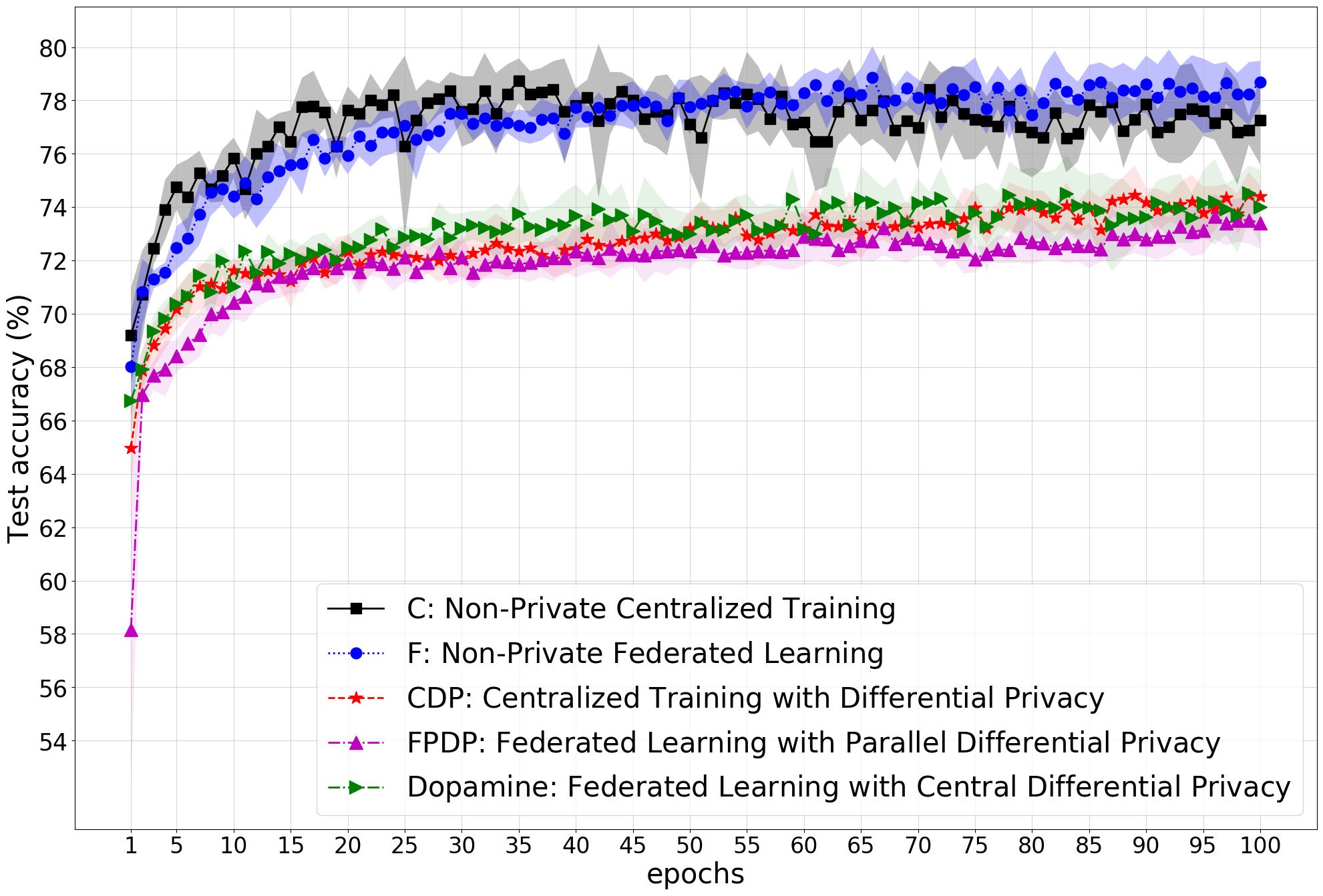}
    \caption{The top plot in Figure~\ref{fig:dr_dataset_1}, but starting from epoch 1.}
    \label{fig:fig_dr_dataset_2}
\end{figure*}

\section{Secure Aggregation}\label{sec:secure_agg}
Secure aggregation helps in providing a computational DP guarantee, even during the training procedure, so that the \agents do not need to trust the \server. We use the Brakerski-Fan-Vercauteren~(BFV) scheme~\cite{cryptoeprint:2012:144} for homomorphic encryption. For security parameter $\lambda$, random element $s\in R_q$, $\Delta = \lfloor \frac{q}{b} \rfloor$, and a discrete probability distribution $\mathbf{\chi} = \mathbf{\chi}(\lambda)$, the BFV scheme is implemented as follows:
\begin{itemize}
    \item \textbf{SecretKeyGen ($1^{\lambda}$):} sample $\mathbf{s}\leftarrow \mathbf{\chi}$, and output $\mathbf{sk} = \mathbf{s}$.
    \item \textbf{PublicKeyGen ($\mathbf{sk}$):} sample $\mathbf{a}\leftarrow R_q$, $e\leftarrow \mathbf{\chi}$, and output
    \begin{align}
        \mathbf{pk} = \left(\left[ -(\mathbf{as} + \mathbf{e}) \right]_q, \mathbf{a} \right).
    \end{align}
    \item \textbf{Encrypt($\mathbf{pk}, \mathbf{m}$):} To encrypt a message $\mathbf{m} \in R_b$, let $\mathbf{p}_0 = \mathbf{pk}[0]$ and $\mathbf{p}_1 = \mathbf{pk}[1]$, and sample $\mathbf{u}, \mathbf{e}_1$, $\mathbf{e}_2 \leftarrow {\chi}$, and output
    \begin{align}
        \mathbf{ct} = \Big( \Big[ \mathbf{p}_0\mathbf{u} + \mathbf{e}_1 + \Delta.\mathbf{m} \Big]_q,  \Big[ \mathbf{p}_1\mathbf{u} + \mathbf{e}_2 \Big]_q \big).
    \end{align}
    \item \textbf{Decrypt($\mathbf{sk}, \mathbf{ct}$):} Set $\mathbf{s}=\mathbf{sk}$, $\mathbf{c}_0=\mathbf{ct}[0]$, and $\mathbf{c}_1=\mathbf{ct}[1]$, and compute
    \begin{align}
        \Big[\Big[ \frac{b.[\mathbf{c}_0 + \mathbf{c}_1\mathbf{s}]_q}{q} \Big]\Big]_b.
    \end{align}
\end{itemize}

We use the open source PySEAL library \cite{titus2018pyseal}, which provides a python wrapper API for the original SEAL library published by Microsoft~\cite{sealcrypto}. We simulate the FL environment over the message passing interface~(MPI), with the rank $0$ process modeling the \server, and the rank $1,\ldots, K$ processes model the \agents that participate in the training. The plaintext modulus $b$ is chosen to be $40961$, while the polynomial modulus $d$ is chosen to be $x^{4096} + 1$. Therefore, each ciphertext can pack $d=4096$ integers with magnitude less than $40961$ for SIMD operations, thus speeding up the communication and computation on encrypted data by $4096$ times. The methodology is as follows:

\begin{enumerate}
    \item {Key generation and distribution:} A public key and secret key are generated by the rank 1 process (\agent 1). The public key is broadcasted to all the involved parties, while the secret key is shared with the \agents but not the \server.
    \item The \server initializes the model to be trained, and sends the model to the \agents in the first round. 
    \item Each \agent performs an iteration on its local model based on Algorithm~\ref{alg:oursys}, at the end of which each \agent has the updated model.
    \item Each active \agent:
    \begin{itemize}
        \item flattens the tensor of model parameters into a 1D array $\mathbf{w}$,
        \item converts the floating point values in the array to 3 digit integers by multiplying each value by $10^3$, and then rounding it to the nearest integer,
        \item partitions the array $\mathbf{w}$ into chunks of $d=4096$ parameters, denoted by $\mathbf{w}_1, \ldots, \mathbf{w}_l$, where $l=\lceil\frac{\text{length}(\mathbf{w})}{4096}\rceil$,
        \item encrypts the chunks $\mathbf{w}_1, \ldots, \mathbf{w}_l$ with the public key, into ciphertexts $\mathbf{ct}_1, \ldots, \mathbf{ct}_l$, respectively, using the batching technique, and sends the list of ciphertexts to the \server.
    \end{itemize}
    \item The  \server adds the received encrypted updates and sends the encrypted aggregated model to the \agents for the next round.
    \item The \agents decrypt the received aggregate array $\mathbf{w}$, scale and convert each integer value to a floating point value of correct precision, divide each value by the number of \agents that were active in the previous round, and finally reshape $\mathbf{w}$ back to the structure of the original model.
    \item Go to step 3.
\end{enumerate} 
At the end of the above procedure, the \agents send the final unencrypted updated models to the \server, which aggregates them and updates the global model for the final time to obtain the final trained model. Notice that we assume hospitals are non-malicious, otherwise one needs to use other techniques, \eg~\cite{corrigan2017prio}. However, methods such as~\cite{corrigan2017prio} use secure multi-party computation, and not homomorphic encryption. The method used in ~\cite{corrigan2017prio} requires multiple non-colluding aggregators (at least one honest aggregator). That is different from our setting, where we consider only one global aggregator. Another alternative is for the hospitals to send secret shares to the other hospitals acting as aggregators. However, that would require the hospitals to be able to communicate with each other, which is not the best solution in dynamic settings.

\end{document}